\documentclass[twoside]{article}

\usepackage[accepted]{aistats2024}
\usepackage[round]{natbib}

\usepackage{amsmath}
\usepackage{amsthm}
\usepackage{amssymb}
\usepackage{graphicx}
\usepackage{booktabs}
\usepackage[utf8]{inputenc}
\usepackage[T1]{fontenc}
\usepackage{xcolor}
\usepackage[pagebackref=true]{hyperref}
\usepackage{microtype}
\usepackage{multirow}
\usepackage{wrapfig}
\usepackage{backref}
\usepackage{enumitem}  % Needs to be loaded late to avoid clashes

\definecolor{gatr-green}{HTML}{419108}
\hypersetup{
    colorlinks=true,
    urlcolor=gatr-green,
    anchorcolor=gatr-green,
    citecolor=gatr-green,
    filecolor=gatr-green,
    linkcolor=gatr-green,
    menucolor=gatr-green,
    linktocpage=true,
    linktoc=all
}

\renewcommand*{\backrefalt}[4]{%
    \ifcase #1 \footnotesize{(Not cited.)}%
    \or        \footnotesize{(Cited on page~#2)}%
    \else      \footnotesize{(Cited on pages~#2)}%
    \fi}
    
\newtheorem{theorem}{Theorem}
\newtheorem{prop}[theorem]{Proposition}
\newtheorem{conjecture}[theorem]{Conjecture}
\newcommand{\R}{\mathbb{R}}
\newcommand{\alg}{\mathcal{G}}
\newcommand{\Pin}{\mathrm{Pin}}
\newcommand{\Spin}{\mathrm{Spin}}
\newcommand{\orth}{\mathrm{O}}
\newcommand{\so}{\mathrm{SO}}
\newcommand{\euc}{\mathrm{E}}
\newcommand{\seuc}{\mathrm{SE}}

\newcommand{\vecop}{\mathrm{vec}}
\newcommand{\spinalg}{\mathfrak{spin}}
\newcommand{\spanop}{\mathrm{span}}

\newcommand{\onegator}{\includegraphics[scale=0.5]{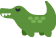}}
\newcommand{\twogators}{\onegator\!\onegator}
\newcommand{\threegators}{\onegator\!\onegator\!\onegator}

\newenvironment{proofsketch}{\proof}{\endproof}

\begin{document}

% If your paper is accepted and the title of your paper is very long,
% the style will print as headings an error message. Use the following
% command to supply a shorter title of your paper so that it can be
% used as headings.
%
\runningtitle{Euclidean, Projective, Conformal: Choosing a Geometric Algebra for Equivariant Transformers}

% If your paper is accepted and the number of authors is large, the
% style will print as headings an error message. Use the following
% command to supply a shorter version of the authors names so that
% they can be used as headings (for example, use only the surnames)
%
%\runningauthor{Surname 1, Surname 2, Surname 3, ...., Surname n}

\twocolumn[

\aistatstitle{Euclidean, Projective, Conformal: \\ Choosing a Geometric Algebra for Equivariant Transformers}

\aistatsauthor{Pim de Haan \And Taco Cohen \And Johann Brehmer}

\aistatsaddress{ Qualcomm AI Research\footnotemark[1] } ]

\begin{abstract}
  % The Geometric Algebra Transformer (GATr \onegator) is a scalable general-purpose architecture for geometric deep learning.  We evaluate different geometric algebras in theory and practice, finding a nuanced trade-off between Euclidean, projective, and conformal algebras.
  % The Geometric Algebra Transformer is a blueprint that allows one to construct a scalable architecture for geometric deep learning given a geometric (or Clifford) algebra. We study versions of this architecture for Euclidean, projective, and conformal algebras, all of which can represent 3D data. We evaluate the resulting architectures in theory and practice. The simplest architecture based on the Euclidean algebra has a smaller symmetry group and is not as sample-efficient, while the projective algebra is not sufficiently expressive. Both the conformal algebra and an improved version of the projective algebra define powerful, performant architectures.
  The Geometric Algebra Transformer (GATr) is a versatile architecture for geometric deep learning based on projective geometric algebra. We generalize this architecture into a blueprint that allows one to construct a scalable transformer architecture given \emph{any} geometric (or Clifford) algebra. We study versions of this architecture for Euclidean, projective, and conformal algebras, all of which are suited to represent 3D data, and evaluate them in theory and practice. The simplest Euclidean architecture is computationally cheap, but has a smaller symmetry group and is not as sample-efficient, while the projective model is not sufficiently expressive. Both the conformal algebra and an improved version of the projective algebra define powerful, performant architectures.
\end{abstract}
\footnotetext[1]{Qualcomm AI Research is an initiative of Qualcomm Technologies, Inc.}
%%%%%%%%%%%%%%%%%%%%%%%%%%%%%%%%%%%%%%%%%%%%
\section{INTRODUCTION}
\label{sec:intro}
%%%%%%%%%%%%%%%%%%%%%%%%%%%%%%%%%%%%%%%%%%%%
Geometric problems require geometric solutions, such as those developed under the umbrella of geometric deep learning~\citep{bronstein2021geometric}.
The primary design principle of this field is equivariance to symmetry groups~\citep{cohen2016group}: network outputs should transform consistently under symmetry transformations of the inputs.
This idea has sparked architectures successfully deployed to problems from molecular modelling to robotics.

In parallel to the development of modern geometric deep learning, the transformer \citep{NIPS2017_3f5ee243} rose to become the de-facto standard architecture across a wide range of domains.
Transformers are expressive, versatile, and exhibit stable training dynamics.
Crucially, they scale well to large systems, mostly thanks to the computation of pairwise interactions through a plain dot product and the existence of highly optimized implementations~\citep{rabe2021self,
dao2022flashattention}.

Only recently have these two threads been woven together.
While different equivariant transformer architectures have been proposed~\citep{Fuchs2020-bw, jumper2021highly, liao2022equiformer}, most involve expensive pairwise interactions that either require restricted receptive fields or limit the scalability to large systems.
\citet{gatr} introduced the Geometric Algebra Transformer (GATr), to the best of our knowledge the first equivariant transformer architecture based purely on dot-product attention. The key enabling idea is the representation of data in the projective geometric algebra. This algebra supports the embedding of various kinds of 3D data and has an expressive invariant inner product.
% a geometric (or Clifford) algebra (GA), having an expressive invariant inner product.
% While multiple suitable GAs exist, \citet{gatr} chose a particular projective algebra.

In this paper, we generalize the GATr architecture to arbitrary geometric (or Clifford) algebras. Given any such algebra, we show how to construct a scalable, equivariant transformer architecture. We focus on the Euclidean and conformal geometric algebra in addition to the projective algebra used by \citet{gatr} and discuss how all three can represent 3D data.

We compare GATr variations based on these three algebras. Theoretically, we study their expressivity, the representation of 3D positions, and the ability to compute attention based on Euclidean distances. In experiments, we compare the architectures on $n$-body modelling tasks and the prediction of wall shear stress on large artery meshes. We also comment on normalization and training stability issues.

All variations of the GATr architecure prove viable, with unique strengths. While the Euclidean architecture is the simplest and most memory-efficient, it has a smaller symmetry group and is less sample-efficient. In its simplest form, the projective algebra is not expressive enough, but it performs well in an improved, more complex version. While the conformal algebra makes normalization more challenging, it offers an elegant formulation of 3D geometry and strong experimental results.

% In this paper, we discuss and compare different geometric algebras that may be used in an equivariant transformer.
% We show how 3D geometric data can be represented in Euclidean and conformal GAs in addition to the projective GA used by \citet{gatr}.
% We then construct new variations of the GATr architectures based on Euclidean and conformal GA representations.
% These architectures are compared both theoretically and in an experiment. All three prove viable, with unique strengths.
% These architecture are compared in terms of their expressivity, memory requirements, training stability, and empirical performance. No architecture emerges as a clear winner, instead we find a nuanced trade-off with different advantages to each representation.

\begin{figure*}[h]
    \centering
    \includegraphics[width=1.5\columnwidth]{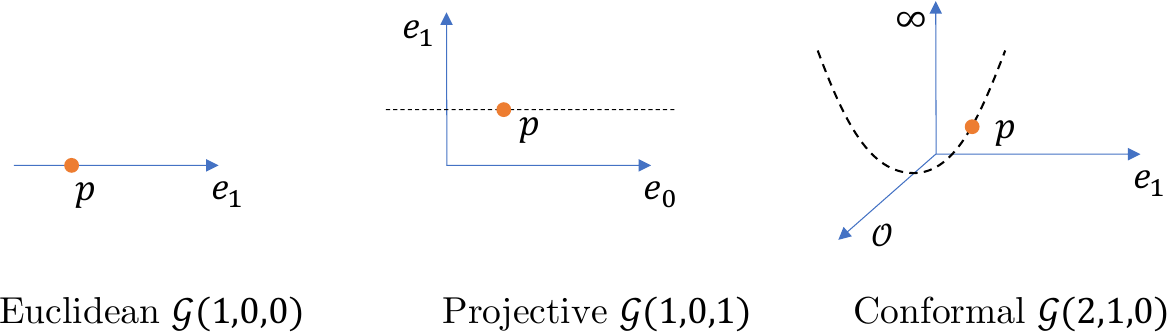}
    \caption{The representations of points in the EGA, PGA and CGA, shown with one spatial dimension for visualization clarity. The dashed lines shows the possible coordinates of points.}
    \label{fig:point-representation}
\end{figure*}

%%%%%%%%%%%%%%%%%%%%%%%%%%%%%%%%%%%%%%%%%%%%
\section{GEOMETRIC ALGEBRAS}
\label{sec:algebras}
%%%%%%%%%%%%%%%%%%%%%%%%%%%%%%%%%%%%%%%%%%%%

\paragraph{Geometric algebra}

We start with a brief introduction to geometric algebra (GA).
An algebra is a vector space that is equipped with an associative bilinear product $V \times V \to V$.

Given a vector space $V$ with a symmetric bilinear inner product, the geometric or Clifford algebra $\alg(V)$ can be constructed in the following way: choose an orthogonal basis $e_i$ of the original $d$-dimensional vector space $V$. Then, the algebra has $2^d$ dimensions with a basis given by elements $e_{j_1}e_{j_2}...e_{j_k}=:e_{j_1j_2...j_k}$, with $1 \le j_1 < j_2 < ... < j_k \le d$, $0 \le k \le d$.
For example, for $V=\R^3$, with orthonormal basis $e_1,e_2,e_3$, a basis for the algebra $\alg(\R^3)$ is
\begin{equation}
1, e_1, e_2, e_3, e_{12}, e_{13}, e_{23}, e_{123}.
\label{eq:euc-ga-basis}
\end{equation}
An algebra element spanned by basis elements with $k$ indices is called a $k$-vector or a vector of \emph{grade} $k$. A generic element whose basis elements can have varying grades is called a \emph{multivector}. A multivector $x$ can be projected to a $k$-vector with the grade projection $\langle x\rangle_k$.

The product on the algebra, called the geometric product, is defined to satisfy $e_ie_j=-e_je_i$ if $i\neq j$ and $e_ie_i=\langle e_i, e_i\rangle$, which by bilinearity and associativity fully specifies the algebra. As an example, for $\alg(\R^3)$, we can work out the following product
\begin{multline}
e_{23}e_{12}=(e_2e_3)(e_1e_2)=(-e_3e_2)(-e_2e_1) \\
=e_3(e_2e_2)e_1
=e_3\langle e_2,e_2 \rangle e_1 = e_3e_1=-e_1e_3=-e_{13}.
\end{multline}

GAs are equipped with a linear bijection $\widehat{e_{j_1j_2...j_k}}=(-1)^k e_{j_1j_2...j_k}$, called the grade involution, a linear bijection $\widetilde{e_{j_1j_2...j_k}}=e_{j_k...j_2j_1}$, called the reversal, an inner product $\langle x, y \rangle=\langle x \widetilde y \rangle_0$, and an inverse $x^{-1}=\widetilde x / \langle x, x \rangle$, defined if the denominator is nonzero.
A group element $u \in \Pin(V)$ acts on an algebra element $x \in \alg(V)$ by (twisted) conjugation: $u[x]=u x u^{-1}$ if $u \in \Spin(V)$ and $u[x]=u \widehat{x}u^{-1}$ otherwise. This action is linear, making $\alg(V)$ a representation of $\Pin(V)$.
From the geometric product, another associative bilinear product can be defined, the wedge product $\wedge$. For $k$-vector $x$ and $l$-vector $y$, this is defined as $x \wedge y=\langle x y \rangle_{k+l}$.

% Given a algebra $\alg(V)$, there is a group $\Pin(V)$ that is generated by the $1$-vectors in the algebra with norm $\pm 1$, and whose group product is the geometric product.
% Its subgroup $\Spin(V)$ is generated by the product of an even amount of $1$-vectors.
% The $\Pin(V)$ group acts linearly on the algebra itself by (twisted) conjugation. See Appendix~\ref{app:ga} for details.

All real inner product spaces are equivalent to a space of the form $\R^{p,q,r}$, with an orthogonal basis with $p$ basis elements that square to +1 ($\langle e_i,e_i\rangle=1$), $q$ that square to -1 and $r$ that square to 0. Similarly, all GAs are equivalent to an an algebra of the form $\alg(\R^{p,q,r})$.
We'll write $\alg(p,q,r):=\alg(\R^{p,q,r}), \Pin(p,q,r):=\Pin(\R^{p,q,r})$.

\paragraph{Geometric algebras for 3D space}

We consider three GAs to model three dimensional geometry. The first is $\alg(3,0,0)$, the \emph{Euclidean} GA (EGA), also known as \emph{Vector} GA. The $k$-vectors have as geometric interpretation respectively: scalar, vectors, pseudovectors, pseudoscalar. A unit vector $x$ in $\Pin(3,0,0)$ represents a mirroring through the plane normal to $x$. Combined reflections generate all orthogonal transformation, making the EGA a representation of $\orth(3)$, or of $\euc(3)$, invariant to translations.
% By the Cartan-Dieudonn\'e theorem, the group $\Pin(3,0,0)$ of combined reflections generate all orthogonal transformations, the group $\orth(3)$. Note that mirroring through $x$ and $-x$ gives the same reflection, making the group $\Pin(3,0,0)$ a \emph{double cover} of $\orth(3)$. Even numbers of reflections generate rotations, so $\Spin(3,0,0)$ covers $\so(3)$. The action of $x$ and $-x$ on the algebra is the same, so the algebra EGA is also a representation of the orthogonal group $\orth(3)$. We can let the algebra also be a representation of the Euclidian group $\euc(3)$ by positing it is invariant to translations.

To represent translation-variant quantities (e.g.\ positions), we can use $\alg(3,0,1)$, the \emph{projective} GA (PGA). Its base vector space $\R^{3,0,1}$ adds to the three Euclidean basis elements, the basis element $e_0$ which squares to 0 (``homogeneous coordinates''). A unit 1-vector in the PGA is written as $v=\mathbf n - \delta e_0$, for a Euclidean unit vector $\mathbf n \in \R^3$, and $\delta \in \R$, and represents a plane normal to $\mathbf n$, shifted $\delta$ from the origin. 
The 2-vectors represent lines and 3-vectors points~\citep{pga-tour}.
The group $\Pin(3,0,1)$ is generated by the unit vectors representing reflections through shifted planes, generating all of $\euc(3)$, including translations.
% Similarly, $\Spin(3,0,1)$ generates $\seuc(3)$ of roto-translations.

The final algebra we consider is $\alg(4,1,0)$, the \emph{conformal} GA (CGA, \citep{Dorst2009}). Its base vector space $\R^{4,0,1}$ adds to the three Euclidean basis elements $e_i$, the elements $e_+$ and $e_-$ which square to $+1$ and $-1$ respectively. Alternatively, it is convenient to choose a non-orthogonal basis $\infty = e_- - e_+$ and $o = (e_-+e_+)/2$, such that $\langle \infty, \infty \rangle = \langle o, o \rangle = 0$ and $\langle \infty, o \rangle = -1$.
% We can construct a map from the PGA to the CGA, $\iota: \alg(3,0,1) \to \alg(4,1,0)$, sending $e_i \mapsto e_i, e_0 \mapsto \infty$, which preserves the geometric product.
Planes in the CGA are represented by a 1-vector $\mathbf n - \delta \infty$, for a Euclidean unit vector $\mathbf n$ and $\delta \in \R$.
The Euclidean group $\euc(3)$ is generated by all such planes, which form a subgroup of $\Pin(4,1,0)$.
% The Euclidean group $\euc(3)$ is not doubly covered by all of $\Pin(4,1,0)$, but by the subgroup $\iota(\Pin(3,0,1)) \subseteq \Pin(4,1,0)$ generated by the 1-vectors $x$ representing planes, which have the coefficient $x_o=0$, or equivalently satisfy $\langle x, \infty \rangle=0$.
The CGA contains a point representation by a null 1-vector $p=o + \mathbf p + \lVert \mathbf p \rVert^2 \infty/2$, for a Euclidean position vector $\mathbf p \in \R^3$.
% This representation satisfies for two points $p, q$ that the inner product $\langle p, q \rangle = -\lVert \mathbf p - \mathbf q \rVert^2/2$ encodes distance.
The different ways in which points are represented in these three algebras are visualized in Figure~\ref{fig:point-representation}.

%%%%%%%%%%%%%%%%%%%%%%%%%%%%%%%%%%%%%%%%%%%%
\section{THE GENERALIZED GEOMETRIC ALGEBRA TRANSFORMER}
\label{sec:gatr}
%%%%%%%%%%%%%%%%%%%%%%%%%%%%%%%%%%%%%%%%%%%%
In this section, we first summarize the prior work, and then discuss how to generalize it from the PGA to the other algebras than model Euclidean geometry.

%%%%%%%%%%%%%%%%%%%%%%%%%%%%%%%%%%%%%%%%%%%%
\subsection{The original Geometric Algebra Transformer}
%%%%%%%%%%%%%%%%%%%%%%%%%%%%%%%%%%%%%%%%%%%%

Our work builds on the Geometric Algebra Transformer (GATr) architecture introduced by \citet{gatr}. This architecture is a transformer architecture~\citep{NIPS2017_3f5ee243} modified in two ways.
First, inputs, outputs, and hidden states consist not only of the usual scalar vector spaces, but also of multiple copies of the projective geometric algebra $\alg(3,0,1)$.
Second, all GATr layers are equivariant with respect to $\euc(3)$, the symmetry group of 3D space.

To satisfy these objectives, the authors construct the most general $\euc(3)$-equivariant linear maps $\alg(3,0,1) \to \alg(3,0,1)$ and modify the nonlinearities and normalization layers to equivariant counterparts. In the MLP, they let the inputs interact via the geometric product and another bilinear interaction, the join. In the attention mechanism, they compute an invariant attention weight between key $k$ and query $q$ -- more on this later.

%%%%%%%%%%%%%%%%%%%%%%%%%%%%%%%%%%%%%%%%%%%%
\subsection{Generalizing GATr to arbitrary algebras}
%%%%%%%%%%%%%%%%%%%%%%%%%%%%%%%%%%%%%%%%%%%%

We now generalize the GATr architecture from the projective algebra to include also the Euclidean and conformal algebras. 
Given a choice of algebra, the generalized GATr architecture uses (many copies of) the algebra as its feature space and is equivariant to $\euc(3)$ transformations.
We will refer to the resulting architecture for the EGA, PGA and CGA as \emph{E-GATr}, \emph{P-GATr}, and \emph{C-GATr}, respectively.

The general GATr construction involves the following modifications to a generic transformer:
\begin{enumerate}
    \item constrain the linear layers to be equivariant,
    \item switch normalization layers and nonlinearities to their equivariant counterparts,
    \item let the inputs to the MLP interact via the geometric product,
    \item compute an invariant attention weight between key $k$ and query $q$ via the algebra's inner product $\langle k, q \rangle$.
\end{enumerate}
We will discuss the construction of equivariant linear maps in Sec.~\ref{sec:equi_linear} and the choice of normalization layers in Sec.~\ref{sec:normalization}.

As discussed in Sec.~\ref{sec:algebras}, three algebras offer natural embeddings for 3D data: the Euclidean algebra $\alg(3,0,0)$, the projective algebra $\alg(3,0,1)$, and the conformal algebra $\alg(4,1,0)$.
We thus construct GATr variants based on these three algebras and refer to them as \emph{E-GATr}, \emph{P-GATr}, and \emph{C-GATr}, respectively.

The projective and conformal algebras are faithful representations of $\euc(3)$.
% P-GATr and C-GATr are therefore able to represent absolute positions and translations.
The Euclidean algebra, however, only transforms by the group $\orth(3)$ of rotations and mirroring. To make E-GATr $\euc(3)$ equivariant, we center inputs, for instance by moving the center of mass to the origin, and make the network $\orth(3)$-equivariant.

The GATr architecture introduced in \cite{gatr} was based on the PGA, but differs from P-GATr in two key ways: the MLP uses the join in addition to the geometric product; and in addition to PGA inner product attention, it uses a map from PGA 3-vectors representing points to CGA 1-vectors representing points and uses the CGA inner product on those. See \cite{gatr} for details.
We refer to this version as improved P-GATr (iP-GATr).

\subsection{Constructing equivariant maps}
\label{sec:equi_linear}
%%%%%%%%%%%%%%%%%%%%%%%%%%%%%%%%%%%%%%%%%%%%
In any geometric algebra, for all $u \in \Pin(V), x, y \in \alg(V)$, we have that $u[xy]=u[x]u[y]$. Hence, the geometric product is equivariant. Also, any $k$-vector transforms into a $k$-vector by the action of $\Pin(V)$, making the grade projection equivariant. As the GA inner product results in a scalar, it is invariant. Furthermore, in the PGA, there is a $\euc(3)$-invariant non-scalar multivector $e_0$. Hence, multiplication with $e_0$ is also $\euc(3)$-equivariant. The EGA has no such invariant multivectors (other than 1).

In \cite{gatr}, it was proven that in the EGA and PGA, all $\euc(3)$-equivariant linear maps can be constructed from linear combinations of geometric product, grade projections and invariant multivectors. To generalize this to include the CGA, in this paper we use a numerical approach to finding equivariant maps.

Let $\alg(V)$ denote the EGA, PGA or CGA. As discussed above, we have an action of the Euclidean group $\euc(3)$ on the algebra, here denoted as a group representation $\rho$, so that for each $g \in \euc(3)$, $\rho(g): \alg(V) \to \alg(V)$ is a linear bijection, respecting the group multiplication structure of $\euc(3)$. Any map $\phi: \alg(V) \to \alg(V)$ is said to be equivariant if for any $g\in \euc(3)$, the following equation is satisfied:
\[
\rho(g) \circ \phi = \phi \circ \rho(g)
\]

If $\phi$ is a linear map, it is equivalently a vector $\vecop(\phi)$ in the vector space $\alg(V) \otimes \alg(V)^*$, where $\alg(V)^*$  is the dual vector space, the vector space of linear maps $\alg(V) \to \R$. This vector space is also equipped with a $\euc(3)$ action $\rho \otimes \rho^*$, where $\rho^*(g):=\rho(g^{-1})^T$ is the called the representation adjoint to $\rho$. Hence, the equivariance constraint of $\phi: \alg(V) \to \alg(V)$ is equivalent to invariance of $\vecop(\phi) \in \alg(V) \otimes \alg(V)^*$, thus satisfying for each $g \in \euc(3)$:
\begin{align}
&(\rho \otimes \rho^*)(g) \vecop(\phi) = \phi \nonumber\\
\iff &((\rho \otimes \rho^*)(g) - 1)\vecop(\phi) = 0
\label{eq:equi-invariance}
\end{align}

This constraint can be solved by sampling sufficiently many $g\in \euc(3)$, row-stacking the matrices $((\rho \otimes \rho^*)(g) - 1)$ and numerically computing its null-space. However, this may require impractically many sampler, and thus computational cost. Later, we'll discuss equivariant multilinear maps, for which this issue is even more pressing. A more efficient approach, as also discussed for generic Lie groups by \cite{Finzi2021-vj}, is to solve the constraint via the Lie algebra. Please see the Appendix for more details.

In the GAs we consider, for any rototranslation $g \in \seuc(3) \subseteq \Spin$, there is a bivector $X$ that represents an infinitesimal transformation, or Lie algebra element: $\exp(X)=g$, where we use the GA exponential that is defined through the Taylor series:
\[
\exp(x)=1 + x + \frac{1}{2!}x^2 + \frac{1}{3!}x^3 + ... 
\]

Filling in $g=\exp(X)$ and collecting linear terms, we find the action of the infinitesimal $X$ on the algebra, which is a Lie algebra representation $d\rho(X)$ sending $v \mapsto d\rho(X)(v) = Xv-vX$. Similarly, from the invariance constraint in Eq.~\eqref{eq:equi-invariance}, we can collect the linear terms in $X$ and obtain an infinitesimal equivariance constraint:
\begin{equation}
d(\rho \otimes \rho^*)(X) \vecop(\phi)=0
\label{eq:lie-alg-contraint}
\end{equation}
where $d(\rho \otimes \rho^*)(X)=d\rho(X) \otimes 1 - 1 \otimes d\rho^T(X)$. This constraint should be statisfied for all $X$ that generate $\seuc(3)$ in the algebra. For the EGA and PGA, these are all bivectors, and for the CGA these are the bivectors generated by $e_1,e_2,e_3,\infty$. The Lie algebra constraint Equation~\ref{eq:lie-alg-contraint} is linear in $X$, so we can require it just on the 6 dimensional basis of the $\seuc(3)$-generating bivectors. For full $\euc(3)$-equivariance including mirroring, we add one additional mirror constraint as in Equation~\ref{eq:equi-invariance}.

Applying this strategy to the EGA, PGA and CGA, and studying the resulting null-space, we find that the pattern found by \cite{gatr} generalizes to the CGA: all equivariant linear maps are linear combinations of grade projections and multiplication with invariant multivectors, which are $e_0$ for the PGA and $\infty$ for the CGA. Thus, for parameters $\alpha, \beta, \gamma, \delta$, these can be parameterized as follows.
For the EGA, we find
\[
\phi(x)=\sum_{k=0}^3 \alpha_k \langle x \rangle_k.
\]
For the PGA, we find
\[
\phi(x)=\sum_{k=0}^4 \alpha_k \langle x \rangle_k + \sum_{k=1}^4 \beta_k \langle e_0 x \rangle_k.
\]
Finally, for the CGA, we find
\begin{align*}
 \phi(x)=\quad
&\sum_{k=0}^5 \alpha_k \langle x \rangle_k \\
+ &\sum_{k=1}^5 \beta_k \langle \infty \langle x \rangle_k \rangle_{k-1} \\
+ &\sum_{k=0}^4 \gamma_k \langle \infty \langle x \rangle_k \rangle_{k+1} \\
+ &\sum_{k=1}^4 \delta_k \infty\langle \infty \langle x \rangle_k \rangle_{k-1}.
\end{align*}
The CGA equivariant linear maps have a different structure from the PGA maps, because when multiplying a $k$-vector in the PGA by $e_0$, one obtains a $k+1$-vector, because $e_0$ has an inner product of 0 with all other vectors. On the other hand, in the CGA, $\langle \infty, o \rangle=-1$, so multiplying a $k$-vector with $\infty$ results in a multivector with grades $k-1$ and $k+1$.
%%%%%%%%%%%%%%%%%%%%%%%%%%%%%%%%%%%%%%%%%%%%
\subsection{Normalization layers}
\label{sec:normalization}
%%%%%%%%%%%%%%%%%%%%%%%%%%%%%%%%%%%%%%%%%%%%

Transformers typically use layer normalization after or, more recently, before the self-attention mechanism and the MLP~\citep{xiong2020layer}.
GATr is no exception and proposes an equivariant modification of LayerNorm, which for $n$ multivector channels is given by
% In GATr, as in typical transformers, LayerNorm is used, which normalizes a collection of $n$ multivector channels jointly. The obvious equivariant interpretation of LayerNorm in GATr would be:
\[
\alg(p,q,r)^n \to \alg(p,q,r)^n : x \mapsto \frac{x}{\sqrt{\frac{1}{n} \sum_{i=1}^n \langle x^i, x^i \rangle + \epsilon}} \,.
\]
To ensure equivariance, this leaves out the shift to zero mean used typically in normalization.

This approach works when $q=r=0$, as then the inner product is directly related to the magnitude of the multivector coefficients, which the normalization layer is designed to keep controlled. However, for the PGA, with $r=1$, the 8 dimensions containing $e_0$ do not contribute to the inner product, making their magnitudes no longer well-controlled. We found that a reasonably high magnitude of $\epsilon=0.01$ suffices to stabilize training.

For the CGA, with $q=1$, the situation is worse. First, as the inner products can be negative, the channels can cancel each other out. In a first attempt to address this, we add the absolute value around the inner product:
\[
\alg(p,q,r)^n \to \alg(p,q,r)^n : x \mapsto \frac{x}{\sqrt{\frac{1}{n} \sum_{i=1}^n \lvert \langle x^i, x^i \rangle\rvert + \epsilon}}
\]
However, also within one multivector some dimension contribute negatively to the inner product and, for example, a scalar and pseudoscalar can cancel out to give a 0-norm (null) multivector. The coefficients of such a multivector grow by $1/\sqrt \epsilon$ with each normalization layer. Empirically, we found that setting $\epsilon=1$ stabilizes training, but harms model performance.

Instead, we found it beneficial to use the following norm in the CGA, which applies the absolute value around each multivector grade separately:
\begin{align*}
\alg(p,q,r)^n &\to \alg(p,q,r)^n \\
x &\mapsto \frac{x}{\sqrt{\frac{1}{n} \sum_{i=1}^n \sum_{k=0}^5 \lvert \langle \langle x^i \rangle_k, \langle x^i\rangle_k \rvert \rangle + \epsilon}}
\end{align*}
This approach mostly addressed stability concerns. However, due to the fact that we still cannot fully control the magnitude of the coefficients, we found it necessary to train C-GATr at 32-bit floating-point precision, whereas the other GATr variants trained well at 16-bit precision (\texttt{bfloat16}).

%%%%%%%%%%%%%%%%%%%%%%%%%%%%%%%%%%%%%%%%%%%%
\section{THEORETICAL COMPARISON}
\label{sec:theory}
%%%%%%%%%%%%%%%%%%%%%%%%%%%%%%%%%%%%%%%%%%%%
We analyze the GATr variants theoretically from three angles: their ability to epress any equivariant multilinear map, their ability to encode absolute positions and their ability to compute attention based on distances.

\subsection{Multilinear expressivity}

To understand better the trade-offs between the GATr variants, we'd like to understand whether they are universal approximators. We will study the slightly simpler question of whether the algebras can express any multilinear map $\alg(V)^l \to \alg(V)$, a map from $l$ multivectors to one multivector, linear in each of the inputs. 
First, we study the case of non-equivariant maps, proven in
% Appendix~\ref{app:multilinear}.
the Appendix.
% \begin{prop}[informal]
% The EGA and CGA can express any non-equivariant multilinear map with operations within the algebra, such as geometric products, and constant multivectors. However, due to the degeneracy of its inner product, the PGA can only do so if we also include the join operation.
% \end{prop}

\begin{prop}
\label{thm:non-equi-multilinear}
    Let $l \ge 1$.
    \begin{enumerate}
        \item[(1)] If and only if the inner product of $\R^{p,q,r}$ is non-degenerate ($r=0$), any multilinear map $\alg(p,q,r)^l \to \alg(p,q,r)$ can be constructed from addition, geometric products, grade projections and constant multivectors.
    
        \item[(2)] Furthermore, any multilinear map $\alg(p,0,1)^l \to \alg(p,0,1)$ can be constructed from addition, geometric products, the join bilinear, grade projections and constant multivectors.
    \end{enumerate}
\end{prop}
\begin{proofsketch}
    For a non-degenerate geometric algebra $\alg(V)$, the GA inner product is a non-degenerate inner product.
    Hence, after picking a basis $b_i$ of the algebra, any linear map $\phi: \alg(V) \to \alg(V)$ can be written as:
    \[
    \phi(x)=\sum_{ij} \alpha_{ij}b_i \langle x, b_j \rangle
    \]
    for coefficients $\alpha_{ij}\in\R$. This argument easily generalizes from linear maps to multilinear maps.

    On the other hand, if the algebra is degenerate, let $e_0$ denote an orthogonal basis vector that squares to 0. Then consider map $\phi: \alg(V) \to \alg(V)$ sending $e_0$ to the scalar 1, and all orthogonal multivectors to 0. This map can not be expressed by the algebra, as the only way to annihilate the $e_0$ is multiplication by $e_0$, which results in 0, not 1.

    However, in the algebra $\alg(p, 0, 1)$ with the join, any multivector can be outer multiplied into the pseudoscalar, which becomes a scalar when joined with 1, from which any multivector can be constructed.
\end{proofsketch}

Thus, the EGA and CGA can express any non-equivariant non-linear map with just the geometric product as bilinear operation, while the PGA requires the join.

For GATr, we are primarily interested in equivariant maps. Here, we don't have a theoretical result, but a conjecture.
% , which we numerically verify up to $l=4$
% (see the Appendix).
% (Appendix~\ref{app:numerical}).
\begin{conjecture}
    Let $l \ge 2$.
    For the EGA and the CGA, and not for the PGA, any $\euc(3)$-equivariant (resp. $\seuc(3)$-equivariant) multilinear map $\alg(p,q,r)^l \to \alg(p,q,r)$ can be constructed out of a combination of the geometric product, grade projection and invariant multivectors. For PGA, any $\seuc(3)$-equivariant multilinear map can be expressed when additionally using the join.
\end{conjecture}

The approach of numerically finding linear maps as described in Section~\ref{sec:equi_linear} can be easily extended to multilinear maps, see the Appendix for details. Comparing those to maps constructable within the algebra, we were able to numerically verify the conjecture up to $l=4$.
These results suggest that the EGA and CGA, and PGA with the join are sufficiently expressive, while the PGA without the join is not.

\subsection{Absolute positions}
The PGA and CGA can represent the absolute position of points: multivectors that are invariant to exactly one rotational $\so(3)$ subgroup of $\seuc(3)$ -- rotations around that point. In contrast, the multivectors of the EGA are invariant to translations, so it can represent directions but not positions. A typical work-around, which we use in our EGA experiments, is to not use absolute positions, but positions relative to some special point, such as the center of mass of a point cloud, which are translation-invariant. This has as downside that the interactions between point-pairs depends on the center of mass. Alternatively, positions can be treated not as generic features in the network, but get special treatment, so that only the position difference between points is used. However, this design decision precludes using efficient dot-product attention in transformers \citep{gatr}.

One difference in the absolute point representations in the PGA and CGA, is that the PGA trivector represent oriented points that flip sign under a mirror. This makes it impossible to construct e.g. a mirror-invariant from a point cloud to $\R$. On the other hand, the CGA contains both oriented and unoriented points, so is able to construct mirror-invariant maps from point clouds.

\subsection{Distance-based attention}

It is desirable when using transformers with geometric systems, that the attention weights between objects can be modulated by their distance.
% In geometric systems, the Euclidean distance between objects plays a large role. As GATr computes object interactions only in the attention module, it is desirable that the architecture can express attention weights that depend on the distance between query and key objects, usually in the form of attending more to closer objects.
In GATr, the attention logits are the GA inner product between a key and query multivector.
% In the following we show that such attention can compute distances in the E-GATr, C-GATr, and iP-GATr, but not in the naive P-GATr architecture.

Distance-based attention appears most naturally in the C-GATr architecture. In the CGA, a Euclidean position vector  $\mathbf{p} \in \R^3$ is represented as $p = o + \mathbf p + \lVert \mathbf p \rVert^2 \infty/2$, and inner products between points directly compute the Euclidean distance.
In the E-GATr, using the positions relative to a center of mass, the inner product of a query consisting of three multivectors $ (\lVert \mathbf{q} \rVert^2, 2\mathbf q, 1)$ and a key $(-1, \mathbf k, -\lVert \mathbf{k} \rVert^2)$ computes negative squared distance.

% However, in P-GATr, dot-product attention cannot compute distances, as we prove in Appendix~\ref{app:distance-based}. In iP-GATr, this is addressed by computing CGA points from PGA points, and using the CGA inner product in the attention; see \citet{gatr} for details.
However, in P-GATr, dot-product attention cannot compute distances. We prove a stronger statement: any inner product must be constant in point coordinates. 

\begin{prop}
    Let $\omega: \R^3 \to \alg(3,0,1), \,x \mapsto x_1 e_{032} + x_2 e_{013} + x_3 e_{021} + e_{123}$ be the point representation of the PGA. For all Spin-equivariant maps $\phi, \psi : \alg(3,0,1) \to \alg(3,0,1)$, for positions $x, y \in \R^3$, the inner product $\langle \phi(\omega(x)), \psi(\omega(y)) \rangle$ is constant in both $x$ and $y$.
\end{prop}
\begin{proof}
    The inner product in the PGA is equal to the Euclidean inner product on the Euclidean subalgebra $\alg(3,0,0)$ (given a basis, this is the subalgebra spanned by elements $e_1, e_2, e_3$, but not $e_0$), ignoring the basis elements containing $e_0$.
    Translations act invariantly on the the Euclidean subalgebra.
    Therefore, for any $v \in \alg(3,0, 1)$, if we consider the map $\R^3 \to \R : x \mapsto \langle \phi(\omega(x)), v \rangle$, this map is invariant to translations, and thus constant.
    Filling in $v = \phi(\omega(y))$ proves constancy of $\langle \phi(\omega(x)), \psi(\omega(y)) \rangle$ in $x$. Constancy in $y$ is shown similarly.
\end{proof}

In iP-GATr, this is addressed by computing CGA points from PGA points, and using the CGA inner product in the attention; see \citet{gatr} for details.

%------------------------------------------------------------
\begin{figure}[t]
    \centering%
    \includegraphics[width=0.4\textwidth,clip=true,trim=0 0 0 0.3cm]{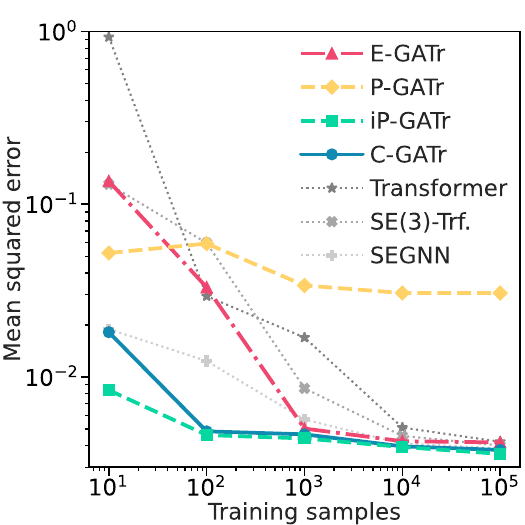}%
    \caption{$n$-body modelling. We show the mean squared error as a function of the number of training samples. We compare E-GATr, P-GATr, iP-GATR, and C-GATr to the equivariant  SE(3)-Transformer~\citep{Fuchs2020-bw} and SEGNN~\citep{Brandstetter2022-hw} as well as to a vanilla transformer.}
    \label{fig:nbody}
\end{figure}
%------------------------------------------------------------

%------------------------------------------------------------
\begin{table}
    \centering%
    \begin{tabular}{lr}
        \toprule
        Method & Approx.\ error\\
        \midrule
        E-GATr & $6.2\;\%$\\
        P-GATr & $7.2\;\%$\\
        iP-GATr & $\mathbf{5.5}\;\%$\\
        C-GATr & $\mathbf{5.5}\;\%$\\
        \midrule
        Transformer & $10.5\;\%$\\
        PointNet++ \citep{qi2017pointnet} & $12.3\;\%$\\
        GEM-CNN \citep{de2020gauge} & $7.7\;\%$\\
        \bottomrule
    \end{tabular}
    \caption{Arterial wall-shear-stress estimation. We show the mean approximation error in percent on the prediction of arterial wall shear stress~\citep{suk2022mesh}. We compare E-GATr, P-GATr, iP-GATr, and C-GATr.
    As baselines, we show the Transformer results from \citet{gatr}, and two baselines from \citet{suk2022mesh}.}
    \label{tbl:experiments}
\end{table}
%------------------------------------------------------------

%%%%%%%%%%%%%%%%%%%%%%%%%%%%%%%%%%%%%%%%%%%%
\section{EXPERIMENTS}
\label{sec:experiments}
%%%%%%%%%%%%%%%%%%%%%%%%%%%%%%%%%%%%%%%%%%%%

We empirically compare the variants in a $n$-body modelling experiment and a hemodynamic estimation task.

\subsection{$n$-body modelling}

We first benchmark the GATr variants on an $n$-body modelling problem. Given masses, initial positions, and initial velocities of 16 point masses interacting with Newtonian gravity, the goal is to predict the final positions after 100 time steps. To make the problem more challenging, we consider a dataset in which each sample has a variable number of clusters, each with a variable number of bodies.

Figure~\ref{fig:nbody} shows the prediction error as a function of the number of training samples. We find that the E-GATr, iP-GATr, and C-GATr models achieve an excellent performance when trained on sufficient data, outperforming or matching the equivariant baselines SE(3)-Transformer~\citep{Fuchs2020-bw} and SEGNN~\citep{Brandstetter2022-hw} and a vanilla transformer. The naive P-GATr does not perform well, a consequence of its fundamentally limited expressivity, discussed in the previous section.

The C-GATr and iP-GATr model are more sample efficient than all baselines and the E-GATr model. We attribute this to their larger symmetry group: they are equivariant with respect to any combination of translations and rotations, while E-GATr is only equivariant with respect to the much smaller group of rotations around the center of mass.

%------------------------------------------------------------
\begin{table*}[t!]
    % \vskip-12pt
    \setlength{\tabcolsep}{4pt}
    \footnotesize
    \centering%
    \begin{tabular}{l llll}
        \toprule
         & E-GATr & P-GATr & iP-GATr & C-GATr \\
        \midrule
        Simplicity & \threegators & \twogators & \onegator & \twogators \\
        Representational \ richness & \onegator & \twogators & \twogators & \threegators \\
        Expressivity & \threegators & \onegator & \threegators & \threegators \\
        \midrule
        Memory & \threegators & \twogators & \twogators & \onegator \\
        Stability & \threegators & \threegators & \threegators & \onegator \\
        Performance & \twogators & \onegator & \threegators & \threegators \\
        \bottomrule
    \end{tabular}
    \caption{Algebras ranked from \threegators\ (best) to \onegator\ (worst) along theoretical qualities (top) and empirical observations (bottom). Alligator figure from Twemoji library by Twitter licensed under CC-BY 4.0
}
    \label{tbl:ranking}
\end{table*}
%------------------------------------------------------------

\subsection{Arterial wall-shear-stress estimation}

Next, we test the GATr variants on a more complex problem: predicting the wall shear stress exerted by  blood flow on the arterial wall, using a benchmark dataset proposed by \citet{suk2022mesh}.
This is a challenging problem for machine learning because the geometric objects are complex---the artery wall is parameterized as a mesh of around 7000 nodes---and the dataset only consists of 1600 meshes.
We describe the experimental setup in more detail in
% Appendix~5.
% % Appendix~\ref{app:experiment}.
the Appendix.

Table~\ref{tbl:experiments} shows our results, with baseline results taken from \cite{gatr} and \cite{suk2022mesh}. All GATrs outperform the baselines.
C- \& iP-GATr which use distance-aware attention, perform best.
We found that C-GATr can suffer from instabilities, as discussed in Sec.~\ref{sec:normalization}.
% see Appendix~\ref{sec:normalization}.

%%%%%%%%%%%%%%%%%%%%%%%%%%%%%%%%%%%%%%%%%%%%
\section{RELATED WORK}
\label{sec:related}
%%%%%%%%%%%%%%%%%%%%%%%%%%%%%%%%%%%%%%%%%%%%

\paragraph{Geometric deep learning}
Constructing neural networks that are equivariant to symmetry groups \citep{cohen2016group} is a cornerstone of contemporary geometric deep learning  \citep{bronstein2021geometric}. Particularly related are the methods that process 3D point clouds in a manner equivariant to the Euclidean symmetries of translation, rotations, and, if desired, mirroring. Typically, these employ linear message passing or transformer architectures  \citep{thomas2018tensor, Fuchs2020-bw, Satorras2021-hf, Brandstetter2022-hw, Batatia2022-ab, Batzner2022-zr, Frank2022-fu}.

\paragraph{Geometric algebras in deep learning}
Geometric (or Clifford) algebras were first conceived in the 19th century \cite{grassmann1844lineale, clifford1878applications} and have been used widely in quantum physics \citep{dirac}. 
More recently, geometric algebras have gained a devoted following in computer graphics \cite{DorstFontijneMann07}. 

The application to GAs in neural networks is not new \citep{bayro1996new}, but has recently experienced a resurgence. \cite{brandstetter2022clifford} uses EGA networks to study differential equations, while \cite{ruhe2023geometric} use the EGA and PGA for message passing. Neither of those architectures is equivariant, however. Equivariant GA networks were proposed by \cite{ruhe2023clifford} using a message passing architecture, and \cite{gatr} using a transformer architecture based on the PGA. This paper extends the equivariant GA transformer to other algebras related to 3D Euclidean geometry.

%%%%%%%%%%%%%%%%%%%%%%%%%%%%%%%%%%%%%%%%%%%%
\section{CONCLUSION}
%%%%%%%%%%%%%%%%%%%%%%%%%%%%%%%%%%%%%%%%%%%%

% Overall, we find that iP-GATr and C-GATr perform similarly, and better than E-GATr and P-GATr. C-GATr is simpler than iP-GATr, as the geometric product and inner product suffice to make a powerful transformer architecture. However, the CGA multivector is the largest, making the computational cost higher than PGA. Also, iP-GATr is more stable.
% E-GATr is the cheapest to compute, but performance may suffer from its use of relative positions. In Table \ref{tbl:ranking}, we assign each variant a score.

The geometric algebra transformer is a powerful method to build $\euc(3)$ equivariant models that scale to large problems due to the transformer backend. In this work, we have generalized the original GATr model, which was based on the projective geometic algebra, to new geometric algebras: the Euclidean and conformal algebras. This construction involved finding the equivariant linear maps and effective normalization layers. From a theoretical analysis of the GATr variants, we found that the Euclidean E-GATr and conformal C-GATr have sufficient expressivity, due to the non-degeneracy of the algebra, while the projective P-GATr does not. Addition of the join bilinear operation, as was done in the original improved projective iP-GATr, can address these issues at the cost of additional complexity in the model. E-GATr can not represent translations or absolute positions, and thus must rely on centering to be $\euc(3)$ equivariant. This reduces the symmetry group and thus sample efficiency. In our experiments, we find that E-GATr has the lowest computational cost, but indeed tends to overfit faster. P-GATr lacks expressivity and thus doesn't perform well, while the original iP-GATr and C-GATr perform best. Of these, C-GATr enjoys the simplicity of just relying on geometric products, while iP-GATr needs the complexity of the join bilinear operation, as well as a hand-crafted attention method. On the other hand, iP-GATr appears more stable in training than C-GATr.
Overall, we find a nuanced trade-off between the variants, which we score in Table \ref{tbl:ranking}.

%%%%%%%%%%%%%%%%%%%%%%%%%%%%%%%%%%%%%%%%%%%%

% \clearpage
\bibliography{refs}

\begin{thebibliography}{30}
\providecommand{\natexlab}[1]{#1}
\providecommand{\url}[1]{\texttt{#1}}
\expandafter\ifx\csname urlstyle\endcsname\relax
  \providecommand{\doi}[1]{doi: #1}\else
  \providecommand{\doi}{doi: \begingroup \urlstyle{rm}\Url}\fi

\bibitem[Bronstein et~al.(2021)Bronstein, Bruna, Cohen, and Veli{\v
  c}kovi{\'c}]{bronstein2021geometric}
Michael~M Bronstein, Joan Bruna, Taco Cohen, and Petar Veli{\v c}kovi{\'c}.
\newblock Geometric deep learning: Grids, groups, graphs, geodesics, and
  gauges.
\newblock 2021.

\bibitem[Cohen and Welling(2016)]{cohen2016group}
Taco Cohen and Max Welling.
\newblock Group equivariant convolutional networks.
\newblock In \emph{International conference on machine learning}, pages
  2990--2999. PMLR, 2016.

\bibitem[Vaswani et~al.(2017)Vaswani, Shazeer, Parmar, Uszkoreit, Jones, Gomez,
  Kaiser, and Polosukhin]{NIPS2017_3f5ee243}
Ashish Vaswani, Noam Shazeer, Niki Parmar, Jakob Uszkoreit, Llion Jones,
  Aidan~N Gomez, {\L}ukasz Kaiser, and Illia Polosukhin.
\newblock Attention is all you need.
\newblock In \emph{Advances in Neural Information Processing Systems},
  volume~30, 2017.

\bibitem[Rabe and Staats(2021)]{rabe2021self}
Markus~N Rabe and Charles Staats.
\newblock Self-attention does not need {$O(n^2)$} memory.
\newblock \emph{arXiv:2112.05682}, 2021.

\bibitem[Dao et~al.(2022)Dao, Fu, Ermon, Rudra, and
  R{\'e}]{dao2022flashattention}
Tri Dao, Dan Fu, Stefano Ermon, Atri Rudra, and Christopher R{\'e}.
\newblock {FlashAttention}: Fast and memory-efficient exact attention with
  {IO}-awareness.
\newblock \emph{Advances in Neural Information Processing Systems},
  35:\penalty0 16344--16359, 2022.

\bibitem[Fuchs et~al.(2020)Fuchs, Worrall, Fischer, and Welling]{Fuchs2020-bw}
Fabian~B Fuchs, Daniel~E Worrall, Volker Fischer, and Max Welling.
\newblock {SE(3)-Transformers}: {3D} {Roto-Translation} equivariant attention
  networks.
\newblock In \emph{Advances in Neural Information Processing Systems}, 2020.

\bibitem[Jumper et~al.(2021)Jumper, Evans, Pritzel, Green, Figurnov,
  Ronneberger, Tunyasuvunakool, Bates, {\v{Z}}{\'\i}dek, Potapenko,
  et~al.]{jumper2021highly}
John Jumper, Richard Evans, Alexander Pritzel, Tim Green, Michael Figurnov,
  Olaf Ronneberger, Kathryn Tunyasuvunakool, Russ Bates, Augustin
  {\v{Z}}{\'\i}dek, Anna Potapenko, et~al.
\newblock Highly accurate protein structure prediction with alphafold.
\newblock \emph{Nature}, 596\penalty0 (7873):\penalty0 583--589, 2021.

\bibitem[Liao and Smidt(2022)]{liao2022equiformer}
Yi-Lun Liao and Tess Smidt.
\newblock Equiformer: Equivariant graph attention transformer for 3d atomistic
  graphs.
\newblock \emph{arXiv:2206.11990}, 2022.

\bibitem[Brehmer et~al.(2023)Brehmer, de~Haan, Behrends, and Cohen]{gatr}
Johann Brehmer, Pim de~Haan, S{\"o}nke Behrends, and Taco Cohen.
\newblock Geometric algebra transformers.
\newblock In \emph{Advances in Neural Information Processing Systems},
  volume~37, 2023.
\newblock URL \url{https://arxiv.org/abs/2305.18415}.

\bibitem[Dorst and De~Keninck()]{pga-tour}
Leo Dorst and Steven De~Keninck.
\newblock A guided tour to the plane-based geometric algebra {PGA}.
\newblock \url{https://bivector.net/PGA4CS.pdf}.
\newblock URL \url{https://bivector.net/PGA4CS.pdf}.
\newblock Accessed: 2023-4-28.

\bibitem[Dorst et~al.(2009)Dorst, Mann, and Fontijne]{Dorst2009}
Leo Dorst, Stephen Mann, and Daniel Fontijne.
\newblock \emph{Geometric Algebra for Computer Science}.
\newblock 2009.

\bibitem[Finzi et~al.(2021)Finzi, Welling, and Wilson]{Finzi2021-vj}
Marc Finzi, Max Welling, and Andrew~Gordon Wilson.
\newblock A practical method for constructing equivariant multilayer
  perceptrons for arbitrary matrix groups.
\newblock April 2021.
\newblock URL \url{http://arxiv.org/abs/2104.09459}.

\bibitem[Xiong et~al.(2020)Xiong, Yang, He, Zheng, Zheng, Xing, Zhang, Lan,
  Wang, and Liu]{xiong2020layer}
Ruibin Xiong, Yunchang Yang, Di~He, Kai Zheng, Shuxin Zheng, Chen Xing,
  Huishuai Zhang, Yanyan Lan, Liwei Wang, and Tieyan Liu.
\newblock On layer normalization in the transformer architecture.
\newblock In \emph{International Conference on Machine Learning}, pages
  10524--10533. PMLR, 2020.

\bibitem[Brandstetter et~al.(2022{\natexlab{a}})Brandstetter, Hesselink,
  van~der Pol, Bekkers, and Welling]{Brandstetter2022-hw}
Johannes Brandstetter, Rob Hesselink, Elise van~der Pol, Erik~J Bekkers, and
  Max Welling.
\newblock Geometric and physical quantities improve {E(3)} equivariant message
  passing.
\newblock In \emph{International Conference on Learning Representations},
  2022{\natexlab{a}}.

\bibitem[Qi et~al.(2017)Qi, Yi, Su, and Guibas]{qi2017pointnet}
Charles~Ruizhongtai Qi, Li~Yi, Hao Su, and Leonidas~J Guibas.
\newblock Pointnet++: Deep hierarchical feature learning on point sets in a
  metric space.
\newblock \emph{Advances in Neural Information Processing Systems}, 30, 2017.

\bibitem[De~Haan et~al.(2021)De~Haan, Weiler, Cohen, and Welling]{de2020gauge}
Pim De~Haan, Maurice Weiler, Taco Cohen, and Max Welling.
\newblock Gauge equivariant mesh {CNNs}: Anisotropic convolutions on geometric
  graphs.
\newblock \emph{International Conference on Learning Representations}, 2021.

\bibitem[Suk et~al.(2022)Suk, de~Haan, Lippe, Brune, and
  Wolterink]{suk2022mesh}
Julian Suk, Pim de~Haan, Phillip Lippe, Christoph Brune, and Jelmer~M
  Wolterink.
\newblock Mesh neural networks for se (3)-equivariant hemodynamics estimation
  on the artery wall.
\newblock \emph{arXiv:2212.05023}, 2022.

\bibitem[Thomas et~al.(2018)Thomas, Smidt, Kearnes, Yang, Li, Kohlhoff, and
  Riley]{thomas2018tensor}
Nathaniel Thomas, Tess Smidt, Steven Kearnes, Lusann Yang, Li~Li, Kai Kohlhoff,
  and Patrick Riley.
\newblock Tensor field networks: Rotation-and translation-equivariant neural
  networks for 3d point clouds.
\newblock \emph{arXiv:1802.08219}, 2018.

\bibitem[Satorras et~al.(2021)Satorras, Hoogeboom, and
  Welling]{Satorras2021-hf}
V{\'\i}ctor~Garcia Satorras, Emiel Hoogeboom, and Max Welling.
\newblock {E(n)} equivariant graph neural networks.
\newblock In Marina Meila and Tong Zhang, editors, \emph{Proceedings of the
  38th International Conference on Machine Learning}, volume 139 of
  \emph{Proceedings of Machine Learning Research}, pages 9323--9332. PMLR,
  2021.

\bibitem[Batatia et~al.(2022)Batatia, Kovacs, Simm, Ortner, and
  Csanyi]{Batatia2022-ab}
Ilyes Batatia, David~Peter Kovacs, Gregor N~C Simm, Christoph Ortner, and Gabor
  Csanyi.
\newblock {MACE}: Higher order equivariant message passing neural networks for
  fast and accurate force fields.
\newblock In \emph{Advances in Neural Information Processing Systems}, 2022.

\bibitem[Batzner et~al.(2022)Batzner, Musaelian, Sun, Geiger, Mailoa,
  Kornbluth, Molinari, Smidt, and Kozinsky]{Batzner2022-zr}
Simon Batzner, Albert Musaelian, Lixin Sun, Mario Geiger, Jonathan~P Mailoa,
  Mordechai Kornbluth, Nicola Molinari, Tess~E Smidt, and Boris Kozinsky.
\newblock {E(3)}-equivariant graph neural networks for data-efficient and
  accurate interatomic potentials.
\newblock \emph{Nat. Commun.}, 13\penalty0 (1):\penalty0 2453, May 2022.

\bibitem[Frank et~al.(2022)Frank, Unke, and Muller]{Frank2022-fu}
Thorben Frank, Oliver~Thorsten Unke, and Klaus~Robert Muller.
\newblock So3krates: Equivariant attention for interactions on arbitrary
  length-scales in molecular systems.
\newblock October 2022.

\bibitem[Grassmann(1844)]{grassmann1844lineale}
Hermann Grassmann.
\newblock \emph{Die lineale Ausdehnungslehre}.
\newblock Otto Wigand, Leipzig, 1844.

\bibitem[Clifford(1878)]{clifford1878applications}
William~Kingdon Clifford.
\newblock {Applications of Grassmann's Extensive Algebra}.
\newblock \emph{Amer. J. Math.}, 1\penalty0 (4):\penalty0 350--358, 1878.

\bibitem[Dirac and Fowler(1928)]{dirac}
Paul Adrien~Maurice Dirac and Ralph~Howard Fowler.
\newblock The quantum theory of the electron.
\newblock \emph{Proceedings of the Royal Society of London. Series A,
  Containing Papers of a Mathematical and Physical Character}, 117\penalty0
  (778):\penalty0 610--624, 1928.
\newblock \doi{10.1098/rspa.1928.0023}.
\newblock URL
  \url{https://royalsocietypublishing.org/doi/abs/10.1098/rspa.1928.0023}.

\bibitem[Dorst et~al.(2007)Dorst, Fontijne, and Mann]{DorstFontijneMann07}
Leo Dorst, Daniel Fontijne, and Stephen Mann.
\newblock \emph{Geometric Algebra for Computer Science: An Object-oriented
  Approach to Geometry}.
\newblock Morgan Kaufmann Series in Computer Graphics. Morgan Kaufmann,
  Amsterdam, 2007.
\newblock ISBN 978-0-12-369465-2.

\bibitem[Bayro-Corrochano et~al.(1996)Bayro-Corrochano, Buchholz, and
  Sommer]{bayro1996new}
Eduardo Bayro-Corrochano, Sven Buchholz, and Gerald Sommer.
\newblock A new self-organizing neural network using geometric algebra.
\newblock In \emph{Proceedings of 13th International Conference on Pattern
  Recognition}, volume~4, pages 555--559. IEEE, 1996.

\bibitem[Brandstetter et~al.(2022{\natexlab{b}})Brandstetter, Berg, Welling,
  and Gupta]{brandstetter2022clifford}
Johannes Brandstetter, Rianne van~den Berg, Max Welling, and Jayesh~K Gupta.
\newblock Clifford neural layers for {PDE} modeling.
\newblock \emph{arXiv:2209.04934}, 2022{\natexlab{b}}.

\bibitem[Ruhe et~al.(2023{\natexlab{a}})Ruhe, Gupta, de~Keninck, Welling, and
  Brandstetter]{ruhe2023geometric}
David Ruhe, Jayesh~K Gupta, Steven de~Keninck, Max Welling, and Johannes
  Brandstetter.
\newblock Geometric clifford algebra networks.
\newblock \emph{arXiv:2302.06594}, 2023{\natexlab{a}}.

\bibitem[Ruhe et~al.(2023{\natexlab{b}})Ruhe, Brandstetter, and
  Forr\'e]{ruhe2023clifford}
David Ruhe, Johannes Brandstetter, and Patrick Forr\'e.
\newblock Clifford group equivariant neural networks.
\newblock \emph{arXiv:2305.11141}, 2023{\natexlab{b}}.

\end{thebibliography}

%%%%%%%%%%%%%%%%%%%%%%%%%%%%%%%%%%%%%%%%%%%%%%%%%%%%%%%%%%%%
\clearpage

\appendix
\onecolumn
\section*{APPENDIX}
% %%%%%%%%%%%%%%%%%%%%%%%%%%%%%%%%%%%%%%%%%%%%
% \section{Additional details on geometric algebras}
% \label{app:ga}
% %%%%%%%%%%%%%%%%%%%%%%%%%%%%%%%%%%%%%%%%%%%%
% GAs are equipped with a linear bijection $\widehat{e_{j_1j_2...j_k}}=(-1)^k e_{j_1j_2...j_k}$, called the grade involution, a linear bijection $\widetilde{e_{j_1j_2...j_k}}=e_{j_k...j_2j_1}$, called the reversal, an inner product $\langle x, y \rangle=\langle x \widetilde y \rangle_0$, and an inverse $x^{-1}=\widetilde x / \langle x, x \rangle$, defined if the denominator is nonzero.
% A group element $u \in \Pin(V)$ acts on an algebra element $x \in \alg(V)$ as $u[x]=u x u^{-1}$ if $u \in \Spin(V)$ and $u[x]=u \widehat{x}u^{-1}$ otherwise. This action is linear, making $\alg(V)$ a representation of $\Pin(V)$.

% From the geometric product, another associative bilinear product can be defined, the wedge product $\wedge$. For $k$-vector $x$ and $l$-vector $y$, this is defined as $x \wedge y=\langle x y \rangle_{k+l}$.
% %%%%%%%%%%%%%%%%%%%%%%%%%%%%%%%%%%%%%%%%%%%%

%%%%%%%%%%%%%%%%%%%%%%%%%%%%%%%%%%%%%%%%%%%%
\section{CONSTRUCTING GENERIC MULTILINEAR MAPS}
\label{app:multilinear}
%%%%%%%%%%%%%%%%%%%%%%%%%%%%%%%%%%%%%%%%%%%%
% \setcounter{prop}{0}
\begin{prop}
\label{thm:app-non-equi-multilinear}
    Let $l \ge 1$.
    \begin{enumerate}
        \item[(1)] If and only if the inner product of $\R^{p,q,r}$ is non-degenerate ($r=0$), any multilinear map $\alg(p,q,r)^l \to \alg(p,q,r)$ can be constructed from addition, geometric products, grade projections and constant multivectors.
    
        \item[(2)] Furthermore, any multilinear map $\alg(p,0,1)^l \to \alg(p,0,1)$ can be constructed from addition, geometric products, the join bilinear, grade projections and constant multivectors.
    \end{enumerate}
\end{prop}
\begin{proof}
    \emph{Proof of (1), ``$\Rightarrow$'':}
    First, let $r=0$. Then let $e_i$ be an orthogonal basis of $\R^{p,q,0}$ where each $e_i$ squares to $\pm 1$. This gives a basis $e_{\mathbf i}$, with multi-index $\mathbf i \in 2^{p+q}$, of the algebra $\alg(p,q,0)$. This basis is also orthogonal and each element $e_{i_1i_2...i_k}$ squares to $\langle e_{i_1i_2...i_k}, e_{i_1i_2...i_k} \rangle = \langle e_{i_1i_2...i_k} \widetilde{e_{i_1i_2...i_k}} \rangle_0 =e_{i_1}e_{i_2}...e_{i_k}e_{i_k}...e_{i_2}e_{i_1}= \prod_k \langle e_k, e_k \rangle = \pm 1$.
    
    % , making the algebra a non-degenerate inner product space.
    Now, let $\phi: \alg(p,q,0) \to \alg(p,q,0)$ be any linear map. For each basis element of the algebra, let $x_{\mathbf i}:=\phi(e_\mathbf i) / \langle e_\mathbf i, e_\mathbf i \rangle$. Then $\phi$ can then be written as:
    \[
        \psi(w) = \sum_{\mathbf i \in 2^{p+q+r}} x_{\mathbf i} \langle w \, \widetilde{e_\mathbf i} \rangle_0
    \]
    It is easy to see that for any basis element $e_\mathbf i$, $\phi(e_\mathbf i)=\psi(e_\mathbf i)$, hence the linear maps coincide.

    For a multilinear map $\phi:\alg(p,q,0)^l \to \alg(p,q,0)$, a similar construction can be made:
    \begin{align*}
        \phi(w_1, ..., w_l) = \sum_{\mathbf i_1 \in 2^{p+q+r}} ...\sum_{\mathbf i_l \in 2^{p+q+r}} x_{\mathbf i_1, ..., \mathbf i_l} \langle w_1 \, \widetilde{e_{\mathbf i_1}} \rangle_0 ... \langle w_l \, \widetilde{e_{\mathbf i_l}} \rangle_0 \\
        \text{with}
        \quad
        x_{\mathbf i_1, ..., \mathbf i_l}=\frac{\phi(e_{\mathbf i_1}, ..., e_{\mathbf i_l})}{\langle e_{\mathbf i_1}, e_{\mathbf i_1}\rangle ... \langle e_{\mathbf i_l}, e_{\mathbf i_l}\rangle}
    \end{align*}

    \emph{Proof of (1), ``$\Leftarrow$'':}
    Let $r > 0$. Let $e_0 \in \R^{p,q,r}$ denote a nonzero radical vector, meaning that for all $x \in \R^{p,q,r}$, $\langle e_0, x \rangle =0$. Consider the multilinear map $\phi:\alg(p,q,r)^l \to \alg(p,q,r)$ sending input $(e_0, ..., e_0) \mapsto 1$ and all other inputs to 0. This map can not be constructed from within the algebra. To see this, consider any nonzero $k$-vector $e_0 \wedge y$ for a $(k-1)$-vector $y$. The only way of mapping $e_0 \wedge y$ to a scalar involves multipication with $e_0$, which results in a zero scalar component.

    \emph{Proof of (2):}
    Now consider the projective algebra $\alg(p, 0, 1)$ equipped with the join $\vee$, a bilinear operation $\alg(p,0,1) \times \alg(p, 0, 1) \to \alg(p, 0, 1)$ mapping algebra basis elements $e_{\mathbf i} \vee e_{\mathbf j}$ to $\pm e_{\mathbf k}$, where $\mathbf k$ contains all indices that occur in both $\mathbf i$ and $\mathbf j$, as long as all $p+1$ indices are present as at least once in either $\mathbf i$ or $\mathbf j$. Otherwise, $e_{\mathbf i} \vee e_{\mathbf j}=0$. See \cite{pga-tour} for details. In particular, the join satisfies $e_{012...p} \vee 1 = 1$.
    
    With the join in hand, any linear map $\phi : \alg(p, 0, 1) \to \alg(p, 0, 1)$ can be written as:
    \[
    \psi(w) = \sum_{\mathbf i \in 2^{p+1}} x_{\mathbf i} \langle (w \wedge e_{\setminus \mathbf i}) \vee 1 \rangle_0
    \]
    where $x_{\mathbf i}:=\phi(e_\mathbf i)$ and $e_{\setminus \mathbf i}$ contains all indices absent in $\mathbf i$, in an order such that $e_{\mathbf i} \wedge e_{\setminus \mathbf i}=e_{012...p}$. For any basis element $e_\mathbf j$, $\langle (e_\mathbf j \wedge e_{\setminus \mathbf i}) \vee 1 \rangle_0=1$ if $\mathbf j = \mathbf i$ and 0 otherwise, because if $\mathbf j$ lacks any index in $\mathbf i$, the join yields a zero, and if it $\mathbf j$ has any indices not in $\mathbf i$, the join results in a non-scalar, which becomes zero with the grade projection. Therefore, $\psi(e_\mathbf i)=\phi(e_\mathbf i)$ for all basis elements $e_\mathbf i$, and the linear maps are equal. As before, this construction easily generalizes to multi-linear maps.
\end{proof}
%%%%%%%%%%%%%%%%%%%%%%%%%%%%%%%%%%%%%%%%%%%%

%%%%%%%%%%%%%%%%%%%%%%%%%%%%%%%%%%%%%%%%%%%%
\section{NUMERICALLY COMPUTING EQUIVARIANT MULTIlinear maps}
\label{app:numerical}
\label{app:equi-lin}
%%%%%%%%%%%%%%%%%%%%%%%%%%%%%%%%%%%%%%%%%%%%
\subsection{Lie group equivariance constraint solving via Lie algebras}
First, let's discuss in generality how to solve group equivariance constraints via the Lie algebra, akin to \cite{Finzi2021-vj}.

Let $G$ be a Lie group, $\mathfrak g$ be its algebra.
% $(\rho_1, V_1)$ and $(\rho_2, V_2)$ be group representations.
Let $\exp: \mathfrak g \to G$ be the Lie group exponential map.

A group representation $(\rho, V)$ induces a Lie algebra representation:
$d\rho : \mathfrak{g} \to \mathfrak{gl}(V)$, linearly sending $X \in \mathfrak g$ to a linear map $d\rho(X): V \to V$, satisfying $\rho(\exp(X))=\exp(d\rho(X))$, where the latter $\exp$ is the matrix exponential.

Given a real Lie algebra representation $(\rho, V)$, there is a dual representation $(\rho^*, V^*)$ satisfying $\rho^*(g)=\rho(g^{-1})^T$. It is easy to see that $d\rho^*(X)=-d\rho(X)^T$.

For two group representations $(\rho_1, V_1)$ and $(\rho_2, V_2)$, there is a tensor representation $(\rho_1 \otimes \rho_2, V_1 \otimes V_2)$ with Lie algebra representation $d(\rho_1 \otimes \rho_2) = 1_{V_1} \otimes d\rho_2 + d\rho_1 \otimes 1_{V_2}$.

$(\rho_1, V_1)$ and $(\rho_2, V_2)$, a linear map $\phi: V_1 \to V_2$ is equivariant if and only if $\phi$ is invariant to the group representation $\rho_2 \otimes \rho_1^*$, when flattening $\vecop(\phi)\in V_2 \otimes V_1^*$: for all $g \in G$,
\[
\rho_2(g) \phi = \phi \rho_1(g) \iff (\rho_2 \otimes \rho_1^*)(g)\vecop(\phi) = \vecop(\phi)
\]

% Given a representation $(\rho, V)$
Any Lie group $G$ is equal to a semi-direct product $G^0 \rtimes D$, for $G^0 \subseteq G$ the subgroup connected to the identity and $D$ a discrete group. Let $B$ be a set of basis elements of the Lie algebra. Then $\exp(\mathrm{span}(B)) = G^0$.

% Note that $\$
% To find a vector $v \in V$ invariant
First, consider a connected Lie group $G^0$, and a basis $B$ of the Lie algebra, and a representation $(\rho, V)$. Then
\begin{align*}
&&\forall g \in G^0,\, \rho(g)v=v \\
\iff && \forall X \in \mathfrak g, \, \rho(\exp(X))v=v \\
\iff && \forall X \in \mathfrak g, \, \exp(d\rho(X))v=v \\
\iff && \forall X \in \mathfrak g, \, d\rho(X)v=0 \\
\iff && \forall X \in B, \, d\rho(X)v=0 \\
\end{align*}
where in for the final step, we note that $d\rho$ is linear, so linearly dependent algebra vectors generate linearly dependent csontraints, and just constraining by a basis of the algebra suffices.

To test invariance to a non-connected Lie group, we need to additionally constrain for the discrete group $D$, generated by $D' \subset D$, leading to:
\[
\forall g \in G\, \rho(g)v = v \iff \begin{cases}
    d\rho(X)v=0 &\forall X \in B \\
    (\rho(g)-1_V)v=0 &\forall g \in D' \\
\end{cases}
\]

% The constraints can be collected into a 
% Each constra
If $V$ is $d$-dimensional There are thus $|B|+|D'|$ $d \times d$ matrices and $v$ needs to be in the null space of each of these, or equivalently in the null space of the concatenated $((|B|+|D'|)d) \times d$ matrix. This can be done numerically via e.g. $\texttt{scipy.linalg.nullspace}$. When $\rho$ can be decomposed into subrepresentations $(\rho_a \oplus \rho_b, V_a \oplus V_b)$, the invariant vectors can be found separately, making computing the null space more efficient.

Combining the framing of equivariance as invariance, and finding invariant vectors via a null space, we can find the linear equivariant maps $\phi: V_1 \to V_2$ by finding the nullspace of:
\[
\forall g \in G\, \rho_2(g)\phi = \phi \rho_1(g) \iff \begin{cases}
    (d\rho_2\otimes 1_{V_1^*}-1_{V_2} \otimes d\rho_1^T)(X)\vecop{\phi}=0 &\forall X \in B \\
    ((\rho_2\otimes \rho_1^*)(g)-1_V)\vecop{\phi}=0 &\forall g \in D' \\
\end{cases}
\]

To find equivariant multilinear maps $\phi: V_{i_1} \otimes V_{i_2} \otimes ... \otimes V_{i_l} \to V$, we simply set $\rho_1 = \rho_{i_1} \otimes \rho_{i_2} \otimes ... \otimes \rho_{i_l}$, with $d\rho_1 = d\rho_{i_1} \otimes 1_{V_{i_2}} \otimes ... + 1_{V_{i_1}} \otimes d\rho_{i_2} \otimes ... + ...$.

\subsection{GA equivariance solving of linear maps}
For the GA, we'll consider $V=\alg(p,q,r)$ and $\rho(u)(x)=u \widehat{x}u^{-1}$. Define $d\rho(X)(v)=Xv - vX$. Any GA has the exponential map endomorphism, defined through the Taylor series:
\[
\exp : \alg(p,q,r) \to \alg(p,q,r) : x \mapsto 1 + x + \frac{1}{2!}x^2 + \frac{1}{3!}x^3 + ... 
\]

\paragraph{EGA}
Now, for the EGA, the bivectors $\alg(3,0,0)_2$ are the Lie algebra $\spinalg(3,0,0)$ of the connected Lie group $\Spin(3,0,0)$ of even number of reflections. The Lie group exponential map is the GA exponential map. The entire Pin group decomposes as $\Pin(3,0,0) = \Spin(3,0,0)\rtimes \{1, e_1\}$.
The bivectors have a basis $\spinalg(3,0,0)=\alg(3,0,0)_2=\spanop(e_{12}, e_{23}, e_{13})$. Therefore, a linear map $\phi: \alg(3,0,0) \to \alg(3,0,0)$ is equivariant to $\Pin(3,0,0)$, and hence to $\orth(3)$, which it doubly covers, and $\euc(3)$ with trivial action under translation, if and only if:
\begin{align*}
\begin{cases}
    (d\rho\otimes 1-1 \otimes d\rho^T)(X)\vecop{\phi}=0 &\forall X \in \{e_{12}, e_{23}, e_{13}\} \\
    ((\rho\otimes \rho^*)(e_1)-1_V)\vecop{\phi}=0 &
\end{cases}
\end{align*}

Studying the nullspace, we find that all equivariant linear maps can be written as linear combinations of grade projections, giving 4 independent maps:
\[
\phi : \alg(3,0,0) \to \alg(3, 0, 0) : x \mapsto \sum_{k=0}^3 \alpha_k \langle x \rangle_k
\]

\paragraph{PGA}
For the PGA, similarly, $\Pin(3,0,1)$ doubly covers $\euc(3)$. The group $\Spin(3,0,1)$ is its connected subgroup, whose algebra are the bivectors, and the Pin group decomposes as the Spin group and a mirroring. A linear map $\phi: \alg(3,0,1) \to \alg(3,0,1)$ is therefore equivariant to $\Pin(3,0,1)$, and hence $\euc(3)$, if and only if:
\begin{align*}
\begin{cases}
    (d\rho\otimes 1-1 \otimes d\rho^T)(X)\vecop{\phi}=0 &\forall X \in \{e_{12}, e_{23}, e_{13}, e_{01}, e_{02}, e_{03}\} \\
    ((\rho\otimes \rho^*)(e_1)-1_V)\vecop{\phi}=0 &
\end{cases}
\end{align*}

Studying the nullspace, we find that all equivariant linear maps can be written as linear combinations of grade projections and multiplications with $e_0$, leading to 9 independent maps:
\[
\phi : \alg(3,0,1) \to \alg(3, 0, 1) : x \mapsto \sum_{k=0}^4 \alpha_k \langle x \rangle_k + \sum_{k=1}^4 \beta_k \langle e_0 x \rangle_k
\]

This result is in accordance with what was shown analytically in \cite{gatr}.

\paragraph{CGA}
Let $\iota: \alg(3,0,1) \to \alg(4,0,1)$ be the algebra homomorphism with $\iota(e_i)=e_i, \iota(e_0)=\infty$.
For the CGA, $E(3)$ is doubly covered by the subgroup $\iota(\Pin(3,0,1))$ of $\Pin(4,1,0)$, hence a linear map $\phi: \alg(4,1,0) \to \alg(4,1,0)$ is equivariant to $\iota(\Pin(3,0,1))$, and hence $\euc(3)$, if and only if:
\begin{align*}
\begin{cases}
    (d\rho\otimes 1-1 \otimes d\rho^T)(X)\vecop{\phi}=0 &\forall X \in \{e_{12}, e_{23}, e_{13}, e_{\infty 1}, e_{\infty 2}, e_{\infty 3}\} \\
    ((\rho\otimes \rho^*)(e_1)-1_V)\vecop{\phi}=0 &
\end{cases}
\end{align*}

Studying the nullspace, we find that all equivariant linear maps can be written as linear combinations of grade projections and multiplications with $\infty$, giving 20 independent maps in total:
\begin{align*}
\phi : \alg(4,1,0) \to &\alg(4, 1, 0) \\
x \mapsto &\sum_{k=0}^5 \alpha_k \langle x \rangle_k \\
+ &\sum_{k=1}^5 \beta_k \langle \infty \langle x \rangle_k \rangle_{k-1} \\
+ &\sum_{k=0}^4 \gamma_k \langle \infty \langle x \rangle_k \rangle_{k+1} \\
+ &\sum_{k=1}^4 \delta_k \infty\langle \infty \langle x \rangle_k \rangle_{k-1}
\end{align*}

\paragraph{$\seuc(3)$ equivariance}
To consider $\seuc(3)$ equivariance, we just have to be equivariant tot the connected part rotational of the Lie group, so remove the mirror constraint in the above equations. For the EGA, PGA and CGA, we find numerically that the $\seuc(3)$-equivariant maps are the same as the $\euc(3)$-equivariant linear maps, but possibly combined with multiplication with the pseudoscalar: $e_{123}$ for the EGA, $e_{0123}$ for the PGA and $e_{123}\wedge o \wedge \infty$ for the CGA. This is because the pseudoscalar is an invariant, up to a sign flip due to mirroring, thus $\seuc(3)$ invariant.

\subsection{Multilinear equivariant map solving}
To find multilinear equivariant maps efficiently, we found it necessary to separate out the grades. For any geometric algebra, the $\Pin(p,q,r)$ representation decomposes into sum of a representation $(\rho_k, \alg(p,q,r)^k)$ of $k$-vectors, for each grade $k$. Then we use the above procedure to find the equivariant multilinear maps $\phi: \alg(p,q,r)^{i_1} \otimes \alg(p,q,r)^{i_2} \otimes ... \otimes \alg(p,q,r)^{i_l} \to \alg(p,q,r)^o$, taking as inputs an $i_1$-vector, and $i_2$-vector, ..., and an $i_l$-vector and outputting an $o$-vector.

\subsection{Numerically testing expressivity}
In the above subsections, we show how one can compute all equivariant multilinear maps for a given algebra. In the main paper, we stated the following conjecture:
\begin{conjecture}
    Let $l \ge 2$.
    For the EGA and the CGA, and not for the PGA, any $\euc(3)$-equivariant (resp. $\seuc(3)$-equivariant) multilinear map $\alg(p,q,r)^l \to \alg(p,q,r)$ can be constructed out of a combination of the geometric product and $\euc(3)$-equivariant (resp. $\seuc(3)$-equivariant) linear maps. For PGA, any $\seuc(3)$-equivariant multilinear map can be expressed using equivariant linear maps, the geometric product and the join.
\end{conjecture}

To test this, we explicitly construct all linear maps via the algebra. Let $\phi^\alpha_{\mu \nu}$ be a basis for the linear equivariant maps of an algebra, so that for each $\alpha$, $y_a=\sum_b \phi^\alpha_{a b}x_b$ is an equivariant linear map, where roman indices enumerate multivector indices. Also, let $\Phi^\beta_{abc}$ be a basis for the bilinears in the algebra, so that for each $\beta$, $z_a = \sum_{bc}\Phi^\beta_{abc}x_b y_c$ is a bilinear. For most algebras, we'll just consider the geometric product, but for the PGA, we can also consider the join, which is only $\seuc(3)$-equivariant \cite[Prop 7]{gatr}. Then, for example, for $l=2$, all bilinear maps constructable for two inputs $x^1, x^2$ from the linears and bilinears are:
\[
\sum_{bc}\Omega^{\sigma \alpha\beta\gamma\delta}_{abc}x^1_bx^2_c= \sum_{bcdef}\phi^\alpha_{ab}\Phi^\beta_{bcd} \left( \phi^\gamma_{ce}x^{\sigma_1}_e \right) \left(\phi^\delta_{df}x^{\sigma_2}_f\right)
\]
where $\sigma \in S_2$ is a permutation over the two inputs. This approach can be recursively applied to construct any multilinear map from the bilinears and linears. As the algebra is not commutative, we need to take care to consider all permutations of the inputs. For computational efficiency to soften the growth in the number of Greek basis indices, during the reduction for multilinear maps, we apply a singluar value decomposition of the basis of maps, re-express the basis in the smallest number of basis maps.

With this strategy, we were able to verify the above conjecture for $2 \le l \le 4$.

\section{EXPERIMENT DETAILS}
\label{app:experiment}

\paragraph{$n$-body modelling dataset}

We create an $n$-body modelling dataset, in which the task is to predict the final positions of a number of objects that interact under Newtonian gravity given their initial positions, velocities, and velocities.
The dataset is created like the $n$-body dataset described in \citet{gatr}, with one exception: rather than a single cluster of bodies, we create a variable number of clusters, each with a variable number of bodies, such that the total number of bodies in each sample is 16. This makes the problem more challenging. Each cluster is generated as described in \citet{gatr}, and the clusters have locations and overall velocities relative to each other sampled from Gaussian distributions.

\paragraph{Arterial wall-shear-stress dataset}

We use the dataset of human arteries with computed wall shear stress by \citet{suk2022mesh}. We use the single-artery version and focus on the non-canonicalized version with randomly rotated arteries. There are 1600 training meshes, 200 validation meshes, and 200 evaluation meshes, each with around 7000 nodes.

\paragraph{Models and training}

Our GATr variants are discussed in the main paper. We mostly follow the choices used in \citet{gatr}, except for the choice of algebra, attention, and normalization layers. For the linear maps, we evaluated two initialization methods: initialize all basis maps with a Kaiming-like scheme, or initialize the linear maps to be the identity on the algebra, and Kaiming-like in the channels. For iP-GATr and P-GATr, we found that the former worked best, for C-GATr we found the latter to work best and for E-GATr we found no difference.

We choose model and training hyperparameters as in \citet{gatr}, except that for the $n$-body experiments, we use wider and deeper architectures with 20 transformer blocks, 32 multivector channels, and 128 scalar channels.

\paragraph{Baselines}

For the $n$-body modelling experiment, we run Transformer, $SE(3)$-Transformer, and SEGNN experiments, with hyperparameters as discussed in \citet{gatr}.

For the artery experiments, baseline results are taken from \citet{gatr} and \citet{suk2022mesh}.

%%%%%%%%%%%%%%%%%%%%%%%%%%%%%%%%%%%%%%%%%%%%
\end{document}